\documentclass{article}                     

\usepackage[margin=2.5cm]{geometry}

\usepackage[utf8]{inputenc}

\usepackage{subfigure}
\usepackage{graphicx}

\usepackage{amsmath,amssymb,amsthm}
\usepackage{amstext}
\usepackage{dsfont}
\usepackage{xspace}
\usepackage{url}

\usepackage{booktabs}                                     
\usepackage{hyperref}                           
\usepackage{amstext}

\newtheorem{theorem}             {Theorem}
\newtheorem{lemma}      [theorem]{Lemma}

\newtheorem{definition} [theorem]{Definition}

\theoremstyle{remark}

\newcommand{\prob}[1]{\Pr\left(#1\right)}
\newcommand{\expect}[1]{\mathbf{E}\left[#1\right]}

\newcommand{\ab}{\hspace{0.125em}}                        
\newcommand{\ie}{\hbox{i.\ab e.}\xspace}                  

\newcommand{\uar}{uar\xspace}

\usepackage[noend]{algorithmic}
\usepackage{algorithm}

\newcommand{\LO}{\text{\sc LO}\xspace}
\newcommand{\LZ}{\text{\sc LZ}\xspace}

\newcommand{\onemax}{\text{\sc OneMax}\xspace}
\newcommand{\twomax}{\text{\sc TwoMax}\xspace}
\newcommand{\OM}{\text{\sc OneMax}\xspace}

\newcommand{\EA}{\text{(1+1)~EA}\xspace}

\newcommand{\EAL}{\text{(1+$\lambda$)~EA}\xspace}
\newcommand{\lEA}{\EAL}

\newcommand{\mlea}{($\mu$+$\lambda$)~EA\xspace}

\DeclareMathOperator{\uBBC}{uBBC}
\newcommand{\upBBC}[1][\lambda]{\ensuremath{\mathrm{#1\text{-}upBBC}}}

\renewcommand{\epsilon}{\varepsilon}
\newcommand{\R}{\mathbb{R}}
\newcommand{\N}{\mathbb{N}}

\newcommand{\ones}[1]{\left|#1\right|_1}

\newcommand{\Oh}[1]{\mathord{O}\mathord{\left(#1\right)}}

\newcommand{\filt}{\mathcal{F}_t}

\newcommand{\filtcond}[1]{\,;\,{#1}\mid\filt}

\newcommand{\xmin}{x_{\min}}
\newcommand{\xmax}{x_{\max}}

\usepackage[numbers,sort&compress]{natbib}

\renewcommand{\epsilon}{\varepsilon}
\newcommand{\E}[1]{\text{E}\left(#1\right)}
\newcommand{\Prob}[1]{\text{Pr}\left(#1\right)}
\newcommand{\ONEMAX}{\OM}

\allowdisplaybreaks[4]

\usepackage{authblk}

\author[1]{Per~Kristian~Lehre}
\author[2]{Dirk~Sudholt}
\affil[1]{School of Computer Science, University of Birmingham, United Kingdom} 
\affil[2]{Department of Computer Science,
        University of Sheffield, United Kingdom}

\title{Parallel Black-Box Complexity with Tail Bounds}

\begin{document}

\maketitle

\begin{abstract}
We propose a new black-box complexity model for search algorithms evaluating $\lambda$ search points in parallel. The parallel unary unbiased black-box complexity gives lower bounds on the number of function evaluations \emph{every} parallel unary unbiased black-box algorithm needs to optimise a given problem. It captures the inertia caused by offspring populations in evolutionary algorithms and the total computational effort in parallel metaheuristics\footnote{This
    paper significantly extends preliminary results which appeared in
    \cite{Badkobeh2014}.}

We present complexity results for LeadingOnes and OneMax. Our main result is a general performance limit: we prove that on \emph{every} function \emph{every} $\lambda$-parallel unary unbiased algorithm needs at least $\Omega(\frac{\lambda n}{\ln \lambda} + n \log n)$ evaluations to find any desired target set of up to exponential size, with an overwhelming probability.
This yields lower bounds for the typical optimisation time on unimodal and multimodal problems, for the time to find any \emph{local} optimum, and for the time to even get close to any optimum. The power and versatility of this approach is shown for a wide range of illustrative problems from combinatorial optimisation.
Our performance limits can guide parameter choice and algorithm design; we demonstrate the latter by presenting an optimal $\lambda$-parallel algorithm for OneMax that uses parallelism most effectively.
\end{abstract}

\section{Introduction}


Black-box
 optimisation describes a challenging realm of problems where no algebraic model or gradient information is available. The problem is regarded a black box, and knowledge about the problem in hand can only be obtained by evaluating candidate solutions.
General-purpose metaheuristics like evolutionary algorithms, simulated annealing, ant colony optimisers, tabu search, and particle swarm optimisers are well suited for black-box optimisation as they generally work well without any problem-dependent knowledge.

A lot of research has focussed on designing powerful metaheuristics, yet it is often unclear which search paradigm works best for a particular problem class, and whether and how better performance can be obtained by tailoring a search paradigm to the problem class in hand.

Black-box complexity is a powerful tool that describes limits on the efficiency of black-box algorithms.
The black-box complexity of search algorithms captures the difficulty of problem classes in black-box optimisation. It describes the minimum number of function evaluations that \emph{every} black-box algorithm needs to make to optimise a problem from a given class. It provides a rigorous theoretical foundation through capturing limits to the efficiency of all black-box search algorithms, providing a baseline for performance comparisons across all known and future metaheuristics as well as tailored black-box algorithms. Also it prevents algorithm designers from wasting effort on trying to achieve impossible performance.

Many different models of black-box complexities have been developed.
The first black-box complexity model by~\citet{Droste2006BlackBox} makes no restriction on the black-box algorithm. This leads to some
unrealistic results, such as polynomial black-box
complexities of NP-hard problems~\citep{Droste2006BlackBox}.
Subsequent research introduced refined models that restrict the power of black-box algorithms, leading to more realistic results~\citep{Teytaud2006,Doerr2011,Droste2006BlackBox,Doerr2012},
where black-box algorithms can only query for the
relative order of function values of search points
\citep{Teytaud2006,Doerr2011} as well as memory restrictions~\citep{Droste2006BlackBox,Doerr2012} and restrictions on which search points are allowed to be stored~\citep{Doerr2015b,Doerr2018a,Doerr2016}.
\Citet*{Lehre2012} introduced the unbiased black-box model
where black-box algorithms may only use operators without a search bias (see Section~\ref{sec:parallel-black-box-model}). This model initially considered
unary operators (such as mutation) and
was later extended to higher arity operators (such as
crossover) \citep{Doerr2010} and more general search spaces
\citep{Rowe2011}. It also led to the discovery of more efficient EA variants~\citep{Doerr2014}. For further details we refer to the comprehensive survey by~\citet{Doerr2018b}.

A shortcoming of the above models is that they do not capture the
implicit or explicit parallelism at the heart of many common search
algorithms. Evolutionary algorithms (EAs) such as \mlea{}s or
($\mu$,$\lambda$)~EAs generate $\lambda$ offspring in parallel. Using
a large offspring population in many cases can decrease the number of
generations needed to find an optimal solution\footnote{This does not
  hold for all problems; \citet*{Jansen2005a}
  constructed problems where offspring populations drastically
  increase the number of generations.}. However,
the number of function evaluations may increase as evolution can only
act on information from the previous generation. A large offspring
population can lead to inertia that slows down the optimisation
process. Existing black-box models are unable to capture this inertia as they assume all search points being created in sequence.

The same goes for parallel metaheuristics such as island models evolving multiple populations in parallel (see, e.\,g.~\citet{Luque2011}). Parallelisation can decrease the number of generations, or parallel time. But the overall computational effort, the number of function evaluations across all islands, may increase. \Citet{Lassig2013a} used the following notion. Let $T_\lambda$ be the random number of generations an island model with $\lambda$ islands (each creating one offspring) needed to find a global optimum for a given problem. If using $\lambda$ islands can decrease the parallel time by a factor of order $\lambda$, compared to just one island, $\lambda \cdot \E{T_\lambda} = O(\E{T_1})$, this is called a \emph{linear speedup} (with regards to the parallel time, the number of generations). A linear speedups means that the total number of function evaluations, $\lambda \cdot \E{T_\lambda}$, does not increase beyond a constant factor.


Previous work~\citep{Lassig2013a,Lassig2011a,Mambrini2012} considered illustrative problems from pseudo-Boolean optimisation and combinatorial optimisation, showing sufficient conditions for linear speedups. However, the absence of matching lower bounds makes it impossible to determine exactly for which parameters~$\lambda$ linear speedups are achieved.

We provide a parallel black-box model that captures and quantifies the inertia caused by offspring populations of size $\lambda$ and parallel EAs evaluating~$\lambda$ search points in parallel. We present lower bounds on the black-box complexity for the well known \LO problem and for the general class of functions with a unique optimum, revealing how the number of function evaluations increases with the problem size~$n$ and the degree of parallelism,~$\lambda$.
The results complement existing upper bounds~\citep{Lassig2013a}, allowing us to characterise the realm of linear speedups, where parallelisation is effective.

Our lower bound for functions with a unique optimum is asymptotically tight: we show that for the \OM problem, a \lEA with an adaptive mutation rate is an optimal parallel unbiased black-box algorithm. Adaptive mutation rates decrease the expected running time by a factor of $\ln \ln \lambda$, compared to the \lEA with the standard mutation rate~$1/n$ (see~\citet*{Doerr2015}).

The paper extends a previous conference paper~\cite{Badkobeh2014} with parts of the results.
A major novelty in this manuscript is the introduction of black-box complexity results with tail bounds. Existing black-box complexity results only make statements about the \emph{expected} number of evaluations it takes to find a global optimum\footnote{A notable exception is the $p$-Monte Carlo runtime introduced by~\citet{Doerr2015b}, defined as the minimum number of
steps needed in order to find an optimum with probability at least $1-p$.}. However, it is often not clear whether the expectation is a good reflection of the performance observed in practice.
We provide black-box complexity lower bounds that apply with an overwhelming probability.
More precisely, using the notation $\ln^+ x := \max(1, \ln x)$, we show for every target search point~$x^*$ we can choose that \emph{every} $\lambda$\nobreakdash-parallel unary unbiased black-box algorithm needs at least
\begin{equation}
\label{eq:lower-bound}
\max\left\{\frac{c\lambda n}{\ln^+\lambda}, (1-\delta) n \ln n\right\}
= \mathord{\Omega}\mathord{\left(\frac{\lambda n}{\ln^+\lambda} + n \ln n\right)}
\end{equation}
function evaluations to find~$x^*$, with an overwhelming probability\footnote{An overwhelming probability is defined as $1-2^{-\Omega(n^\varepsilon)}$ for some constant~$\varepsilon > 0$.}, where
$c$ is a constant with $c \ge 1/60$.
The leading constant $1-\delta$ in the $n \ln n$ term can be chosen\footnote{The precise result contains a trade-off between the leading constant and the exponent of the overwhelming probability formula, see Theorem~\ref{the:main-result-tail-bounds}.} arbitrarily close to~1. This means that it is practically impossible for any unary unbiased black-box algorithm to find a designated target with less than $\frac{c\lambda n}{\ln^+\lambda}$ or less than $(1-\delta)n \ln n$ evaluations. The latter bound applies to parallel and non-parallel unary unbiased algorithms.



In addition, if the probability of finding a single target~$x^*$ in the stated time is exponentially small, the probability of finding \emph{many} target points is still exponentially small.
This simple union bound argument opens up a range of opportunities for obtaining stronger results that are much more relevant to practice than the state-of-the-art. Our method is powerful and versatile since we can choose any set of target search points, up to an exponential size. This allows for different applications.
\begin{enumerate}
\item Considering global optimisation, our lower bound~\eqref{eq:lower-bound} applies to highly multimodal functions, even allowing for up to exponentially many optima.
Apart from results tailored to specific problem classes~\citep{Doerr2013b}, the only generic black-box complexity results we are aware of apply to functions with one unique global optimum. This innovation is significant as most functions in practice have multiple or many optima.
\item Choosing all local optima as target search points, we also get that for functions with up to exponentially many local optima, every $\lambda$-parallel unary unbiased algorithm needs at least the stated time~\eqref{eq:lower-bound} to find any \emph{local} optimum.
\item Since we can have exponentially many target search points, we can even afford to consider all search points within an almost linear Hamming distance to any local optimum as target. Then our results imply that even the time to get close to any local or global optimum is bounded by~\eqref{eq:lower-bound}.
\end{enumerate}
%
%
%
We demonstrate the applicability and versatility of our main result by deriving the first black-box complexity results for a wide range of illustrative function classes, from synthetic problems (\twomax, \textsc{H-IFF}, \textsc{Jump}$_k$, \textsc{Cliff}) that are very popular in the evolutionary computation literature to classes of benchmark functions~\cite{Jansen2016} and important problems from combinatorial optimisation such as \textsc{Vertex Colouring}, \textsc{MinCut}, \textsc{Partition}, \textsc{Knapsack} and \textsc{MaxSat}.

In addition to providing a solid unifying theoretical foundation for black-box algorithms, we believe that our results are of immediate relevance to practice.
Our black-box complexity with tail bounds gives hard limits on the capabilities of black-box algorithms. These limits can be used to set stopping criteria appropriately, avoiding stopping an algorithm before it has had a chance to come close to local or global optima. They are useful to set parameters such as the offspring population size~$\lambda$: if we have a limited computational budget of $T$ evaluations, \eqref{eq:lower-bound} implies that we must choose $\lambda$ satisfying $\lambda/\ln^+ \lambda \le T/(cn)$ as for larger values $T$ is lower than~\eqref{eq:lower-bound}, meaning that every $\lambda$\nobreakdash-parallel unary unbiased black-box algorithm fails badly with overwhelming probability.
Moreover, our lower bounds can serve as baseline in performance comparisons across various algorithms. And, last but not least, knowing what is \emph{impossible} is vital for guiding the search for the \emph{best possible} algorithm. The feasibility of this approach is demonstrated in this work as we present an optimal $\lambda$-parallel algorithm for \onemax that uses parallelism most effectively.


\section{A Parallel Black-Box Model}
\label{sec:parallel-black-box-model}

Following~\citet{Lehre2012}, we only use unary unbiased variation operators, i.\,e., operators creating a new search point out of one search point. This includes local search, mutation in evolutionary algorithms, but it does not include recombination.
Unbiasedness means that there is no bias towards particular regions of the search space; in brief, for $\{0, 1\}^n$, unbiased operators must treat all bit values $0, 1$ and all bit positions $1, \dots, n$ symmetrically (see~\citet{Lehre2012,Rowe2011} for details). This is the case for many common operators such as standard bit mutation.



Unbiased black-box algorithms query new search points based on the past history of function values, using unbiased variation operators. We define a \emph{$\lambda$-parallel unbiased black-box algorithm} in the same way, with the restriction that in each round $\lambda$ queries are made in parallel (see Algorithm~\ref{alg:parallel-black-box}). We use the abbreviation \uar for \emph{uniformly at random}. These $\lambda$ queries only have access to the history of evaluations from previous rounds; they cannot access information from queries made in the same round.
We refer to these $\lambda$ search points as \emph{offspring} to indicate search points created in the same round.

\begin{algorithm}[htb]
    \caption{$\lambda$-parallel unbiased black-box algorithm} \label{alg:parallel-black-box}
\begin{algorithmic}[1]
\STATE Let $t := 0$. Choose $x^1(0), \dots, x^\lambda(0)$ \uar, compute $f(x^1(0)), \dots, f(x^\lambda(0))$, and let $I(0) := \{f(x^1(0)), \dots, f(x^\lambda(0))\}$.
\REPEAT
\FOR{$1 \le i \le \lambda$}
\STATE Choose an index $0 \le j \le t$ according to $I(t)$.
\STATE Choose an unbiased variation operator $p_v(\cdot \mid x(j))$ according to $I(t)$.
\STATE Generate $x^i(t+1)$ according to~$p_v$.
\ENDFOR
\FOR{$1 \le i \le \lambda$}
\STATE Compute $f(x^i(t))$ and let $I(t) := I(t) \cup \{f(x^i(t))\}$.
\ENDFOR
\STATE Let $t := t + 1$.
\UNTIL{termination condition met}
\end{algorithmic}
\end{algorithm}

This black-box model includes offspring populations in evolutionary algorithms, for example \mlea{}s or ($\mu$,$\lambda$)~EAs (modulo minor differences in the initialisation). It can further model parallel evolutionary algorithms such as cellular EAs with $\lambda$ cells, or island models with $\lambda$ islands, each of which generates one offspring in each generation.

The \lEA maintains the current best search point $x$ and creates $\lambda$ offspring by flipping each bit in~$x$ independently with probability~$p$ (with default $p=1/n$). The best offspring replaces its parent if it has fitness at least $f(x)$.

\begin{algorithm}[htb]
    \caption{\lEA} \label{alg:lEA}
\begin{algorithmic}[1]
\STATE Choose $x$ \uar.
\REPEAT
\FOR{$1 \le i \le \lambda$}
\STATE Create $y_i$ by copying~$x$ and flipping each bit independently with probability $1/n$.
\ENDFOR
\STATE Choose $z \in P_t$ \uar from $\arg\max\{f(y_1), \dots, f(y_\lambda)\}$.
\STATE \algorithmicif{} {$f(z) \ge f(x)$} \algorithmicthen{} $x = z$
\UNTIL{termination condition met}
\end{algorithmic}
\end{algorithm}

\subsection{Parallel black-box complexity}

The \emph{unbiased black-box complexity (uBBC)} of a function class $\mathcal{F}$ is the minimum worst-case runtime among all unbiased black-box algorithms~\citep{Lehre2012} (equivalent to Algorithm~\ref{alg:parallel-black-box} with $\lambda=1$).
The \emph{unbiased $\lambda$-parallel black-box complexity ($\lambda$-upBBC)} of a function class $\mathcal{F}$ is defined as the minimum worst-case number of function evaluations among all unbiased $\lambda$-parallel algorithms satisfying the framework of Algorithm~\ref{alg:parallel-black-box}.


With increasing $\lambda$ access to previous queries becomes more and more restricted. It is therefore not surprising that the black-box complexity is non-decreasing with growing~$\lambda$.
For every family of function classes $\mathcal{F}_n$ and all $\lambda \in \N$,
\begin{align}
\label{eq:hierarchy-of-lambda-BBC}
\uBBC(\mathcal{F}_n) &\; \le \upBBC(\mathcal{F}_n) \le \lambda \cdot \uBBC(\mathcal{F}_n)
\end{align}
as any unbiased algorithm can be simulated by a $\lambda$-parallel unbiased black-box algorithm using one query in each round.


The following lemma shows that the parallel black-box complexity increases with the degree of parallelism, modulo possible rounding issues.
\begin{lemma}
\label{lem:bbc-grows}
For any $\alpha,\beta \in \mathbb{N}$, if $\alpha \le \beta$ then
\[
\upBBC[\alpha]{}(\mathcal{F}_n) \le \frac{\alpha}{\beta}\left\lceil \frac{\beta}{\alpha}\right\rceil\cdot \upBBC[\beta]{}(\mathcal{F}_n)\]
In particular, if $\frac{\beta}{\alpha} \in \mathbb{N}$ then $\upBBC[\alpha]{} \le \upBBC[\beta]{}$.
\end{lemma}
A proof (in the context of distributed black-box complexity) was given in~\cite[Lemma~4]{Badkobeh2015}.

Lemma~\ref{lem:bbc-grows} implies the following for all function classes $\mathcal{F}_n$ (we omit $\mathcal{F}_n$ for brevity):
First, if $\frac{\beta}{\alpha} \in \mathbb{N}$ then $\upBBC[\alpha]{} \le \upBBC[\beta]{}$. Otherwise, $\upBBC[\alpha]{} \le (1+\frac{\alpha}{\beta})\cdot\upBBC[\beta]{}
\le 2\cdot\upBBC[\beta]{} $ because $\lceil \frac{\beta}{\alpha}\rceil \le 1+\frac{\beta}{\alpha}$ and $1+\frac{\alpha}{\beta} \le 2$. In particular, this implies that for all $\alpha < \beta \in \N$,
\begin{equation}
\label{eq:bbc-does-not-decrease}
\upBBC[\beta]{} = \Omega(\upBBC[\alpha]{}).
\end{equation}
We conclude that the $\lambda$-parallel black-box complexity
does not asymptotically decrease with the degree of parallelism, $\lambda = \lambda(n)$. This implies that there is a \emph{cut-off point}
such that for all $\lambda = \Oh{\lambda^*}$ the $\lambda$-parallel unbiased black-box complexity of $\mathcal{F}_n$ is asymptotically equal to the regular unbiased black-box complexity.\footnote{Strictly speaking, we should be writing $\lambda(n) = \Oh{\lambda^*(n)}$ as the degree of parallelism may depend on~$n$. We omit this parameter for ease of presentation. Asymptotic statements always refer to~$n$.}
\begin{definition}
\label{def:cutoff}
A value $\lambda^*$ is a cut-off point if
\begin{itemize}
\item for all $\lambda = \Oh{\lambda^*}$,   $\upBBC{} = \Oh{\uBBC}$ and
\item for all $\lambda = \omega(\lambda^*)$, $ \upBBC{} =  \omega(\uBBC)$.
\end{itemize}
\end{definition}

Such a cut-off point always exists because due to~\eqref{eq:bbc-does-not-decrease} the parallel black-box complexity cannot decrease asymptotically, and values of $\Oh{\uBBC}$ can always be attained for suitable~$\lambda^*$, e.\,g.\ for $\lambda^* := 1$. Furthermore, the $\lambda$-parallel black-box eventually diverges for very large~$\lambda$ (e.\,g. $\lambda = \omega(\uBBC)$) as trivially $\upBBC{} \ge \lambda$.

Note that cut-off points are not unique: if $\lambda^*$ is a cut-off point, then every $\lambda' = \Theta(\lambda^*)$ is also a cut-off point.

A cut-off point determines the realm of linear speedups~\citep{Lassig2013a}, where parallelisation is most effective. Below the cut-off, for an optimal parallel black-box algorithm the number of function evaluations does not increase (beyond constant factors), but the number of rounds decreases by a factor of~$\Theta(\lambda)$. The number of rounds corresponds to the parallel time if all $\lambda$ evaluations are performed on parallel processors. Hence, below the cut-off it is possible to reduce the parallel time proportionally to the number of processors, without increasing the total computational effort (by more than a constant factor).

\section{Parallel Black-Box Complexity of LeadingOnes}
\label{sec:lo}

We consider the function $\LO(x) := \sum_{i=1}^n \prod_{j=1}^i x_j$, counting the number of leading ones in~$x$.
It is an example of a unimodal function where a specific bit needs to be flipped.
Similarly, $\LZ(x)$ counts the number of leading zeros in~$x$.
We first provide a tool for estimating the progress made by $\lambda$ trials, which may or may not be independent. It is based on moment-generating functions (mgf).
\begin{lemma}\label{lemma:mgf-max-bound}
  Given $X_1,\ldots,
  X_\lambda\in\mathbb{N}$, where  $X_i$s are random variables, not necessarily
  independent. Define $X_{(\lambda)}:=\max_{i\in[\lambda]} X_i$, if there exists $\eta,D\geq 0$, such that for all
  $i\in[\lambda]$, it holds $\E{e^{\eta X_i}}\leq D$, then
  $\E{X_{(\lambda)}}\leq (\ln(D\lambda)+1)/\eta$.
\end{lemma}
\begin{proof}
  Note first that for any $i\in[\lambda]$ and $j\in\mathbb{N}$,
  it follows from Markov's inequality that
$    \Pr(X_i\geq j)
    = \Pr(e^{\eta X_i}\geq e^{\eta j})
    \leq e^{-\eta j}\E{e^{\eta X_i}}
    \leq e^{-\eta j}D.
$
Now, let $k:=\ln(D\lambda)/\eta$. It then follows by a union bound that
\begin{align*}
  \E{X_{(\lambda)}}
   =\;&        \sum_{i=1}^\infty \Pr(X_{(\lambda)}\geq i)
   \leq k + \sum_{i=1}^\infty \Pr(X_{(\lambda)}\geq k+i)\\
   \leq\;& k + \sum_{i=1}^\infty\sum_{j=1}^\lambda \Pr(X_{j}\geq k+i)
   \leq k + \sum_{i=1}^\infty \lambda e^{-\eta (k+i)}D \\
   =\;&    k + e^{-\eta k}\frac{D\lambda }{e^\eta-1}
   \leq k + e^{-\eta k} D\lambda/\eta
   =    (\ln(D\lambda)+1)/\eta. \qedhere
\end{align*}
\end{proof}

For the \LO function, the $\lambda$-parallel black-box complexity is
as follows.
\begin{theorem}
\label{the:LO}
Let $\ln^+ x  := \max(1, \ln x)$.
The $\lambda$-parallel unbiased black-box complexity of \LO is
\[
\mathord{\Omega}\mathord{\left(\frac{\lambda n}{\ln^+(\lambda/n)} + n^2\right)} \text{\quad and \quad} \Oh{\lambda n + n^2}.
\]
The cut-off point is
$
\lambda^*_{\LO} = n.
$
The corresponding parallel time for an optimal algorithm is $\mathord{\Omega}\mathord{\left(\frac{n}{\ln^+(\lambda/n)} + \frac{n^2}{\lambda}\right)}$ and $\Oh{n + \frac{n^2}{\lambda}}$.
\end{theorem}
This result solves an open problem from~\citet{Lassig2013a}, confirming that the analysis of the realm of linear speedups for \LO from~\citet{Lassig2013a} is tight.
\begin{proof}[Proof of Theorem~\ref{the:LO}]
The upper bound follows from~\citet[Theorem~1]{Lassig2011a} for a \lEA, as within the context of this bound the \lEA is equivalent to an island model with complete communication topology.

A lower bound $\Omega(n^2)$ follows from~\citet{Lehre2012}, hence the
statement holds for the case $\lambda = O(n)$.
Thus we only need to consider the case $\lambda = \omega(n)$ and to prove a lower bound of $\mathord{\Omega}\mathord{\left(\frac{\lambda n}{\ln^+(\lambda/n)}\right)} = \mathord{\Omega}\mathord{\left(\frac{\lambda n}{\ln(\lambda/n)}\right)}$ for this case.

We proceed by drift analysis. Let the
``potential'' of a search point $x$  be
\[
\max_{0\leq j\leq
  t,1\le i \le \lambda}\{\LO(x^i(j)), \LZ(x^i(j)), n/2\}
\]
and define the potential of the algorithm, $P_t$ at time
$t$ to be the largest potential among all search points produced until
time~$t$.

Assume that the potential in generation $t$ is $P_t=k$.  In any
generation $t$, let $X_i$ for $i\in[\lambda]$ be the indicator
variable for the event that all of the first $k+1$ bit-positions in
individual $i$ are $1$-bits (or $0$-bits). Furthermore, let $Y_i$ be
the number of consecutive 1-bits (respectively 0-bits) from position
$k+2$ and onwards, ie., the number of ``free riders''.

Following the same arguments as in~\citet{Lehre2012}, the probability
that $X_i=1$ is no more than ${1/(k+1)}=O(1/n)$. Defining
$M:=\sum_{i=1}^\lambda X_i$, we therefore have $\E{M}=O(\lambda/n)$.  Each
random variable $Y_i$, $i\in[\lambda]$, is stochastically dominated by
a geometric random variable $Z_i$ with parameter $1/2$.  The expected
progress in potential is therefore
\begin{align*}
  \E{\Delta_{(\lambda)}}
  = \E{\max_{i\in[\lambda]} X_iY_i}
  \leq \E{\max_{i\in[M]} Z_i}.
\end{align*}
The mgf of the geometric random variable $Z_i$ is
$M_{Z_i}(\eta)=1/(2-e^\eta)$. The tower property of the expectation and Lemma~\ref{lemma:mgf-max-bound} with $\eta:=\ln(3/2)$ and $D:=2$
give
\begin{align*}
  \E{\Delta_{(\lambda)}}
  \leq\;& \E{\E{\max_{i\in[M]} Z_i\mid M}}\\
  \leq\;& \E{(\log(DM)+1)/\eta}\\
  \leq\;& (\log(\E{DM})+1)/\eta
  = O(\log(\lambda/n)),
\end{align*}
where the last inequality follows from Jensen's inequality and the
last equality follows from $\log(\lambda/n) = \Omega(1)$. With
overwhelmingly high probability, the initial potential is at least
$n/2$. Hence, by classical additive drift theorems~\citep{He2004}, the
expected number of rounds to reach the optimum is $\Omega(n/\log(\lambda/n))$. Multiplying by $\lambda$ gives the number of function evaluations.
\end{proof}

\section{Parallel Black-Box Complexity of Functions with One Unique Optimum}
\label{sec:unique-optimum}

\Citet*{Jansen2005a} considered the \lEA
and established a cut-off point for~$\lambda$ where the running time
increases from $\Theta(n \log n)$ to $\omega(n \log n)$:
\begin{equation}
\label{eq:cutoff-for-1+lambda-on-OneMax}
\lambda^*_{\text{\lEA on \OM}} = \Theta((\ln n)(\ln \ln n)/(\ln \ln \ln n))
\end{equation}

\Citet*{Doerr2015} presented the following tight bounds for bounded~$\lambda$:
\begin{theorem}[Adapted from~\citet*{Doerr2015}]
\label{the:He-Chen-Yao}
The expected optimisation time of the \lEA on \ONEMAX is
\[
\mathord{\Theta}\mathord{\left(n \cdot \frac{\lambda \log \log \lambda}{\log \lambda} + n \log n\right)}
\]
where the upper bound holds for $\lambda = O(n^{1-\varepsilon})$ and the lower bound holds for $\lambda = O(n)$.
\end{theorem}


We show that the parallel black-box complexity is lower than the bound
from Theorem~\ref{the:He-Chen-Yao} for large~$\lambda$ by a factor of
order $\log \log \lambda$.
\begin{theorem}
\label{the:black-box-complexity-onemax}
For any $\lambda \le e^{\sqrt{n}}$ the $\lambda$-parallel unbiased unary black-box complexity for any function with a unique optimum is at least
\[
\mathord{\Omega}\mathord{\left(\frac{\lambda n}{\ln^+ \lambda} + n \log n\right)}.
\]
This bound is tight for \OM, where the cut-off point is
\[
\lambda^*_{\OM} = \Theta(\log(n) \cdot \log \log n).
\]
The corresponding parallel time for an optimal algorithm is $\mathord{\Omega}\mathord{\left(\frac{n}{\ln^+ \lambda} + \frac{n \log n}{\lambda}\right)}$.
\end{theorem}
Note that the cut-off point is higher than the cut-off point for the
\lEA with the standard mutation rate $p=1/n$
from~\eqref{eq:cutoff-for-1+lambda-on-OneMax} and~\citet{Jansen2005a}.

For the proof we consider the progress made during a round of
$\lambda$ variations in terms of a potential function defined in the following. The following definitions and arguments, including several lemmas shown in the following, will also be used in Section~\ref{sec:tail-bounds} to prove lower bounds that hold with overwhelming probability.

Without loss of generality, we assume that the search point $1^n$ is
  the optimum. Following \citet{Lehre2012}, we assume a ``mirrored''
  sampling process, where every time a bit string~$x$ is queried (including
  in the initial generation), the algorithm queries the complement
  bit string $\overline{x}$ for ``free''. This makes sense as the complement of any bit string can be generated by flipping all bits. Thus we have to consider the progress towards the global optimum as well as the progress towards its complement.
\begin{definition}
\label{def:progress-measures}
Let $s_0^t$ be the minimum number of zeros in all search points queried in all steps up to time~$t$.
For all $s_0^t \le m \le n-s_0^t$ and $r \in \{0, \dots, n\}$ we define the random variable $\Delta_0(s_0^t, m, r) := \max\{0, s_0^t - |y|_0\}$ where $y$ is a random search point obtained by applying unbiased variation with radius~$r$ to a search point with $m$ zeros.
Define $s_1^t$ and $\Delta_1$ symmetrically with respect to the number of ones.

Due to mirrored sampling, we always have $s_0^t = s_1^t$, hence we simply write $s^t$ or just $s$ if we refer to the current point in time.
Then we define the progress in terms of the potential as $\Delta(s, m, r) = \max\{\Delta_0(s, m, r), \Delta_1(s, m, r)\}$.
\end{definition}
Note in particular that for all $z \in \mathbb{N}$ we have
\begin{align}
\label{eq:bound-on-Delta}
\Prob{\Delta(s, m, r) \ge z}
&\le \Prob{\Delta_0(s, m, r) \ge z} + \Prob{\Delta_1(s, m, r) \ge z}
\end{align}

Also note that by symmetry of zeros and ones $\Delta_0(s, m, r)$ has the same distribution as $\Delta_1(s, n-m, r)$, hence it suffices to study the distribution of $\Delta_0$.
We also have for all $s, s \le m \le n-s, r$
\begin{equation}
\label{eq:Delta-symmetry}
\Delta_0(s, m, r) = \Delta_0(s, n-m, n-r)
\end{equation}
as flipping all bits (in the transition from $m$ to $n-m$) and then flipping all but $r$ bits in the variation has the same effect as flipping $r$ bits in the first place.
Hence it suffices to consider $\Delta_0(s, m, r)$ for $s \le m \le n/2$.


Now consider the progress $\Delta_0(s, m, r)$. Let $Z$ be the number of 0\nobreakdash-bits that flipped to~1, then
there are $r-Z$ new 0-bits that were originally~1. Therefore, the
number of 0\nobreakdash-bits in the new generated search point is $m-Z+(r-Z)$
where $Z$ can be described by the hypergeometric distribution with
parameters $n, m$ and $r$.  We only make progress if the number of
0-bits in the new search point is less than $s$. Hence the progress
(decrease in $0$-potential) is
\begin{align*}
\Delta_0(s,m,r)=\;& \max\{ Z-(r-Z)+(s-m),0\}\\
=\;& \max\{ 2Z-r+s-m,0\}.
\end{align*}

We show a tail inequality for hypergeometric variables and use this to derive a progress bound for the 0-potential.
\begin{lemma}
\label{lem:hypergeometric-tail}
Let $Z$ be a hypergeometrically distributed random variable with parameters $n$ (number of balls), $m$ (number of red balls), and $r$ (number of balls drawn). For all ${z \in \N_0}$,
\[
\Prob{Z = z} \le \binom{r}{z} \cdot \frac{m^z}{n^z}
\le \left(\frac{4m}{n}\right)^z
\]
where the second inequality holds for $z \ge r/2$.
\end{lemma}
\begin{proof}
We assume $z \le m$ and $z \le r$ as otherwise $\Prob{Z = z} = 0$.
We further assume $z \ge 1$ as for $z=0$ the probability bound is~1 and the statement is trivial.
Now,
\begin{align}
\Prob{Z = z}
=\;& \binom{m}{z}\binom{n-m}{r-z}/\binom{n}{r}\notag\\
=\;& \frac{m!(n-m)!r!(n-r)!}{z!(m-z)!(r-z)!(n-m-r+z)!n!}\notag\\
=\;& \binom{r}{z} \cdot \frac{m!(n-m)!(n-r)!}{(m-z)!(n-m-r+z)!n!}\label{eq:progress-factorials}.
\end{align}
The fraction can be written as
\allowdisplaybreaks[0]
\begin{align*}
& \frac{m(m-1)\cdot \ldots \cdot (m-z+1)}{n(n-1)\cdot \ldots \cdot(n-z+1)} \cdot 
\frac{(n-m)(n-m-1)\cdot \ldots \cdot(n-m-r+z+1)}{(n-z)(n-z-1) \cdot \ldots \cdot(n-r+1)}
\end{align*}
\allowdisplaybreaks[4]
Since $z \le m$, the second fraction above is at most~1. The first fraction is at most $m^z/n^z$
as $(m-i)/(n-i) \le m/n$ for all $i \in \N$ and $m \le n$.
Plugging this into~\eqref{eq:progress-factorials} and using $\binom{r}{z} \le 2^r \le 2^{2z} = 4^z$ for $z \ge r/2$ yields
\[
\Prob{Z = z} \le \binom{r}{z} \cdot \frac{m^z}{n^z}
\le \left(\frac{4m}{n}\right)^z. \qedhere
\]
\end{proof}

The following lemma shows that for any radius~$r$ the probability of having a progress of~$z$ decreases exponentially with~$z$.
\begin{lemma}\label{lemma:improve-prob}
Let $s$ denote the current 0-potential.
If $s\leq m\leq n/8$, then for all $z\in\mathbb{N}$ and $r\in \{1, \dots, n\}$,
\[
\Prob{\Delta_0(s, m, r) = z} \le \left(\frac{1}{2}\right)^{z/2}.
\]
\end{lemma}
\begin{proof}
Applying Lemma~\ref{lem:hypergeometric-tail} to a hypergeometric random variable $Z$ with parameters $m$ and $r$ we have, for all $z \in \N_0$,
\begin{align*}
\Prob{\Delta_0(s, m, r) = z}
=\;& \Prob{Z =\frac{z + r + m - s}{2}}\\
\le\;& \left(\frac{4m}{n}\right)^{(z+r+m-s)/2} \le \left(\frac{1}{2}\right)^{z/2}. \qedhere
\end{align*}
\end{proof}

The following lemma gives another tail bound that will be used to exclude steps where a search point of potential $m \gg s$ is chosen for variation. The probability of having a positive progress decreases rapidly with growing $m-s$.
\begin{lemma}
\label{lem:applying-Chvatal}
For every $s \le m \le n/2$ and every $r \in \{1, \dots, n\}$
\[
\Prob{\Delta_0(s, m, r) > 0} \le \exp\left(-\frac{(m-s)^2}{2r}\right).
\]
\end{lemma}
\begin{proof}
We use Chv\'{a}tal's tail bound~\citep{Chvatal1979}:
$\Prob{Z\geq \E{Z}+r\delta} \leq
  \exp(-2\delta^2r)$, where $\E{Z}=\frac{rm}{n}$.
 \begin{align*}
 & \Prob{\Delta_0(s,m,r)>0}\\
= \;& \Prob{Z > \frac{r+m-s}{2}}\\
= \;& \Prob{Z > \frac{r m}{n} +r \cdot \left(\frac{r+m-s}{2r}-\frac{m}{n}\right)}\\
\leq\;& \Prob{Z \geq \frac{rm}{n} + r \cdot \left(\frac{m-s}{2r}\right)}\\
\leq\;& \exp\left(-2r \left(\frac{m-s}{2r}\right)^2\right)
= \exp\left(-\frac{(m-s)^2}{2r}\right).\qedhere
 \end{align*}
\end{proof}

Putting all lemmas together shows that the expected progress is at most logarithmic in $\lambda$.
\begin{lemma}
\label{lem:progress-m}
Let $\Delta_0^{(\lambda)} = \Delta_0^{(\lambda)}(s, m_i, r_i)$ be the maximum of $\lambda$  random variables $\Delta_0(s, m_i, r_i)$ for arbitrary $s \le m_i \le n/2$ and $r_i$, $1\le i\le\lambda$.
For $s \le n/16$ we have
$\E{\Delta_0^{(\lambda)}}  =  \Oh {\log(\lambda)}$.
\end{lemma}
\begin{proof}
If $n/8 < m_i \le n/2$ then by Lemma~\ref{lem:applying-Chvatal}
 \begin{align*}
 & \Prob{\Delta_0(s,m_i,r_i)>0} \le e^{-n^2/(512r_i)} \le e^{-\Omega(n)}.
 \end{align*}
This means that the probability of making any progress is
exponentially small, for any~$r_i$. Thus $\E{\Delta_0^{(\lambda)}}$ is maximised if we assume that $m_i\le n/8$ for all~$i$.

Under this assumption, applying Lemma~\ref{lemma:improve-prob}, for all $z \in \N_0$,
\begin{align*}
\Prob{\Delta_0(s, m_i, r_i) = z} \le\;& \left(\frac{1}{2}\right)^{z/2}
\end{align*}
hence $\E{e^{\eta \Delta_0(s, m_i, r_i)}} \le D$
for $\eta := \ln(4/3)$ and $D := 9+6\sqrt{2}$. Applying Lemma~\ref{lemma:mgf-max-bound} proves $ \E{\Delta_0^{(\lambda)}} = \Oh{\log \lambda}$.
\end{proof}

Now we are in a position to prove Theorem~\ref{the:black-box-complexity-onemax}.
\begin{proof}[Proof of Theorem~\ref{the:black-box-complexity-onemax}]
  The upper bound for \OM will be shown later in
  Theorem~\ref{the:upper-bound-lea-adaptive-mutation}.  The lower
  bound $\Omega(n \log n)$ follows from unbiased unary black-box
  complexity~\citep{Lehre2012}. Hence, it suffices to prove
  the lower bound $\Omega(\lambda n/\ln^+ \lambda)$ for $\lambda \ge 3$, where $\ln^+ \lambda$ can be replaced by $\ln \lambda$.


  Consider any $\lambda$-parallel unary unbiased black-box algorithm.
  We grant the algorithm an advantage by revealing all search points with Hamming distance at least $n/16$ to both $0^n$ and $1^n$ at no cost. Hence the potential is always $s \le n/16$.
  Let $\Delta_0^{(\lambda)}$ be the progress due to reduction of the
  $0$\nobreakdash-potential in one step, and $\Delta_1^{(\lambda)}$ be the progress due to reduction of
  the $1$\nobreakdash-potential.
  By virtue of the symmetry of $\Delta_0$ and $\Delta_1$, Lemma~\ref{lem:progress-m} also applies to $\Delta_1^{(\lambda)}$.
  Hence the expected change in
  potential per round is no more than
  \[
  \E{\Delta_0^{(\lambda)}} + \E{\Delta_1^{(\lambda)}} = O(\log \lambda).
   \]
  Hence, by the additive drift theorem~\citep{He2004}, the expected number of rounds
  until one of the search points $0^n$ or $1^n$ is obtained is
  $\Omega(n/\log \lambda)$. Multiplying by~$\lambda$ proves the claim.
\end{proof}

\section{An Optimal Parallel Black-Box Algorithm for OneMax}
\label{sec:optimal-algorithm-onemax}

The following theorem shows that the lower bound on the black-box complexity from Theorem~\ref{the:black-box-complexity-onemax} is tight.
We show that the \lEA has a better optimisation time if the mutation rate is chosen adaptively, according to the current best fitness. This is similar to common ideas from artificial immune systems, particularly the clonal selection algorithm. Adaptive mutation rates for \OM have been studied by~\citet{Zarges2008}, however the standard parameters for the clonal selection algorithm were too drastic to even obtain polynomial running times. Better results were obtained when using a population-based adaptation~\citep{Zarges2009}.

The following result reveals an optimal choice for the mutation rate of the \lEA, depending on~$n$ and $\lambda$.
\begin{theorem}
\label{the:upper-bound-lea-adaptive-mutation}
On OneMax, the expected number of function evaluations of the \lEA with an adaptive mutation rate $p = \max\{\ln(\lambda)/(n\ln(en/i)), \ 1/n\}$, where $i$ is the number of zeros in the current search point, for any $\lambda \le e^{\sqrt{n}}$, is at most
\[
\Oh{\frac{\lambda n}{\ln \lambda} + n \log n}.
\]
The parallel time (number of generations) is $\Oh{\frac{n}{\ln \lambda} + \frac{n \log n}{\lambda}}$.
\end{theorem}
\begin{proof}
Let $i$ be the current number of zeros and $p$ be the mutation rate. The probability of decreasing the number of zeros by any~$k \in \N$ with $k \le i$ is at least
\begin{align*}
\Prob{\Delta \ge k} \ge\;& \binom{i}{k} \cdot p^{k} \cdot (1-p)^{n-k}\\
\ge\;& \frac{i^{k}}{k^k} \cdot p^{k} \cdot (1-p)^{n-k}
= (1-p)^{n-k} \cdot \left(\frac{ip}{k}\right)^{k}.
\end{align*}
Then the probability that one of $\lambda$ offspring will decrease the number of zeros by at least $k$ is at least, using $1-(1-p)^\lambda \ge 1-e^{-p\lambda} \ge 1 - 1/(1+p\lambda) = p\lambda/(1+p\lambda)$,
\begin{align*}
\Prob{\Delta_{(\lambda)} \ge k} \ge 1-(1-\Prob{\Delta \ge k})^\lambda
\ge\;& \frac{\lambda (1-p)^{n-k} \cdot (ip/k)^{k}}{1 + \lambda (1-p)^{n-k} \cdot (ip/k)^{k}}.
\end{align*}
Hence for any $k \le i$ the expected drift is at least
\begin{align*}
\E{\Delta_{(\lambda)}} \ge\;& k \cdot \frac{\lambda (1-p)^{n-k} \cdot (ip/k)^{k}}{1 + \lambda (1-p)^{n-k} \cdot (ip/k)^{k}}.
\end{align*}
For $i > en/\ln \lambda$, which implies $pn > 1$, we set $k := pn = \ln(\lambda)/\ln(en/i)$. We have $k \le i$ since $k \le \ln(\lambda) \le \sqrt{n} \le en/\ln \lambda$.
We use $k := 1$ for $i \le en/\ln \lambda$, the realm where $p=1/n$.
This results in the following drift function~$h$:
\[
h(i) := \begin{cases}
\frac{\lambda (1-1/n)^{n-1} \cdot i/n}{1 + \lambda (1-1/n)^{n-1} \cdot i/n} & \text{if $i \le en/\ln \lambda$}\\
pn \cdot \frac{\lambda (1-p)^{n-pn} \cdot (i/n)^{pn}}{1 + \lambda (1-p)^{n-pn} \cdot (i/n)^{pn}} & \text{otherwise}
\end{cases}
\]
We estimate the number of function evaluations by multiplying the number of generations by~$\lambda$. The number of generations is estimated using Johannsen's variable drift theorem~\citep{Johannsen2010} in the variant from~\citet{Rowe2013}, with the above function~$h$. This gives an upper bound of
\begin{align*}
 \frac{\lambda}{h(1)} + \int_1^n \frac{\lambda}{h(i)} \ \mathrm{d}i
=\;& \frac{1 + \lambda (1-1/n)^{n-1} \cdot 1/n}{(1-1/n)^{n-1} \cdot 1/n} + \lambda \int_1^n \frac{1}{h(i)} \ \mathrm{d}i\\
\le\;& \lambda + en + \lambda \int_1^{en/\ln \lambda} \frac{1}{h(i)} \ \mathrm{d}i + \lambda \int_{en/\ln \lambda}^n \frac{1}{h(i)} \ \mathrm{d}i.
\end{align*}
The first terms are at most
\begin{align*}
\;& \lambda + en + \lambda \int_1^{en/\ln\lambda} \frac{1 + \lambda (1-1/n)^{n-1} \cdot i/n}{\lambda (1-1/n)^{n-1} \cdot i/n} \ \mathrm{d}i\\
\le \;& \frac{\lambda en}{\ln \lambda} + en \left(1 + \int_1^{en/\ln\lambda} \frac{1}{i}\ \mathrm{d}i\right)
\le \frac{\lambda en}{\ln \lambda} + en \cdot (2+\ln n).
\end{align*}
The second integral is bounded as
\begin{align*}
\;& \int_{en/\ln\lambda}^n \frac{1 + \lambda (1-p)^{n-pn} \cdot (i/n)^{pn}}{pn \cdot (1-p)^{n-pn} \cdot (i/n)^{pn}} \ \mathrm{d}i\\
\le\;& \int_{0}^n \frac{\lambda \ln(en/i)}{\ln \lambda} \ \mathrm{d}i + \frac{1}{\ln \lambda} \int_{en/\ln\lambda}^n \frac{\ln(en/i)}{e^{-pn} \cdot (i/n)^{pn}} \ \mathrm{d}i\\
=\;& \frac{2\lambda n}{\ln \lambda} + \frac{1}{\ln \lambda} \int_{en/\ln\lambda}^n \ln(en/i) \cdot (en/i)^{pn} \ \mathrm{d}i\\
=\;& \frac{2\lambda n}{\ln \lambda} + \frac{1}{\ln \lambda} \int_{en/\ln\lambda}^n \ln(en/i) \cdot \lambda \ \mathrm{d}i
\le \frac{3\lambda n}{\ln \lambda}.
\end{align*}
Together, we get an upper bound of
$(3+e)\lambda n/\ln(\lambda) + en \cdot (2+\ln n)$.
\end{proof}


Note that the optimal mutation rate $p = \max\{\ln(\lambda)/(n\ln(en/i)), \ 1/n\}$, in particular the functional relationship between the mutation rate and the current fitness~$i$, is quite hard to guess through experimentation and was only revealed through the present theoretical analysis.
After the result from Theorem~\ref{the:upper-bound-lea-adaptive-mutation} was first published~\citep{Badkobeh2014}, \citet*{Doerr2018} presented a self-adjusting scheme for choosing the mutation rate in the \lEA and showed that it is able to match the upper bound from Theorem~\ref{the:upper-bound-lea-adaptive-mutation} without knowing the functional relationship between the mutation rate and the current fitness.


\section{Tail Bounds}
\label{sec:tail-bounds}

In this section we now show that the lower bound for all $\lambda$\nobreakdash-parallel unbiased unary black-box algorithms from Theorem~\ref{the:black-box-complexity-onemax} holds with high probability. In particular, it also applies to (non-parallel) unbiased unary black-box algorithms, for which only lower bounds on the expectation were known before~\citep{Lehre2012}. Our main result is as follows.
\begin{theorem}
\label{the:main-result-tail-bounds}
For every unary unbiased $\lambda$-parallel black-box algorithm $\mathcal{A}$ and every constant $0 < \delta < 1$, with probability $1-\exp(-\Omega(n^{\delta}/\log n))$ $\mathcal{A}$ does not find any target set of at most $\exp(o(n^{\delta}/\log n))$ search points within time
\[
\max\left\{\frac{\lambda n}{60\ln^+\lambda}, (1-\delta) n \ln n\right\}
= \mathord{\Omega}\mathord{\left(\frac{\lambda n}{\ln^+\lambda} + n \ln n\right)}.
\]
The expected time also satisfies the asymptotic bound.
\end{theorem}

Theorem~\ref{the:main-result-tail-bounds} establishes very general limits to the performance of large classes of algorithms, including mutation-only evolutionary algorithms with standard mutation operator, local search, simulated annealing.
In particular, putting $\delta := 0.01$ (say), Theorem~\ref{the:main-result-tail-bounds} shows that every unary unbiased search algorithm needs to be run for at least $n \ln n$ evaluations as the probability of finding one of few global optima within $0.99n \ln n$ evaluations is overwhelmingly small. The same holds for $\lambda$-parallel unary unbiased algorithms like mutation-only evolutionary algorithms with offspring populations of size~$\lambda$. Here stopping a run before $\lambda n/(60 \ln^+ \lambda)$ evaluations is futile as with overwhelming probability no optimum will have been found yet.

In addition, Theorem~\ref{the:main-result-tail-bounds} makes a statement about a target set of up to exponential size. This means that the lower bounds also apply to functions with many global optima, with respect to the optimisation time, but it can also be used to bound the time to find local optima or any set of high-fitness individuals of size at most $\exp(o(n^{\delta}/\log n))$. Illustrative applications to a broad range of well-known problems will be given in Section~\ref{sec:applications-of-tail-bounds}.

Theorem~\ref{the:main-result-tail-bounds} will be shown by separately showing lower bounds of $\Omega(\lambda n/\log \lambda)$ and $\Omega(n \log n)$ that both hold with overwhelming probability. Throughout this section we again assume ``mirrored'' sampling, i.\,e.\ every queried search point~$x$ also evaluates $\overline{x}$ for free.

\subsection{Lower Bound \boldmath$\Omega(\lambda n/\log \lambda)$ with overwhelming probability}

We start with a bound of $\Omega(\lambda n/\log \lambda)$. Recall from
Definition~\ref{def:progress-measures} that due to mirrored sampling,
we can define the potential as the minimum number zeros, or equivalently
number of ones, in all search points up to time $t$. In order to use
Theorem~\ref{theorem:general-drift-theorem} for a tail bound on the runtime, we
need to study the mgf. of the progress
\begin{align*}
  \Delta^{(\lambda)}(s) := 
                      \max \left\{\Delta^{(\lambda)}_0(s,m,r),\Delta^{(\lambda)}_1(s,m,r)\right\},
\end{align*}
where $\Delta^{(\lambda)}_0(s,m,r)$ is the maximal progress in the
0-potential, and $\Delta^{(\lambda)}_1(s,m,r)$ is the maximal progress
in the 1-potential, given current potential $s$, where the selected
search point has $m$ 0-bits, respectively 1-bits, and $r$ bits are
flipped.



\begin{lemma}\label{lemma:mgf-bound}
Let $s$ denote the current potential.
  If $s\leq \frac{n}{8}$ and
  $\gamma:=\ln\left(\frac{3}{4}\sqrt{2}\right)$, then
  $\expect{e^{\gamma\Delta^{(\lambda)}(s)}} \leq 8\lambda$.
\end{lemma}
\begin{proof}
  As noted in Definition~\ref{def:progress-measures} and
  (\ref{eq:Delta-symmetry})
  \begin{align*}
   \Delta_1(s,m,r) = \Delta_0(s,n-m,r) = \Delta_0(s,m,n-r).
  \end{align*}
  Hence, by a union bound
  \begin{multline*}
    \prob{\Delta(s,m,r)=z}
                             \leq \prob{\Delta_0(s,m,r)=z} +
                             \prob{\Delta_1(s,m,r)=z}\\
                            = \prob{\Delta_0(s,m,r)=z} +
                            \prob{\Delta_0(s,m,n-r)=z}
                            \leq 2^{1-z/2}
\end{multline*}
where the last inequality follows by Lemma \ref{lemma:improve-prob}.
We now have
  \begin{align*}
    \expect{e^{\gamma\Delta^{(\lambda)}(s, m, r)}}
    & = \sum_{z=0}^\infty \prob{\Delta^{(\lambda)}=z}e^{\gamma z},\\
\intertext{by a union bound over $\lambda$ parallel runs}
    & \leq \sum_{z=0}^\infty \lambda\max_{r\in[n],s\leq m}\prob{\Delta(s,m,r)=z}e^{\gamma z}\\
\intertext{the definition of $\gamma$ give}
    & \leq \lambda \sum_{z=0}^\infty 2\left(\frac{1}{2}\right)^{z/2}\left(\frac{3}{4}\sqrt{2}\right)^z\\
    & =    \lambda \sum_{z=0}^\infty 2\left(\frac{3}{4}\right)^z = 8\lambda.\qedhere
  \end{align*}
\end{proof}


\begin{theorem}
\label{the:tail-bound-lambda-term}
  If  $\lambda\geq 1$, then  $\Prob{T<\frac{\lambda n}{60\ln^+\lambda}}=e^{-\Omega(n)}$.
\end{theorem}
\begin{proof}
  Following the proof of
  Theorem~\ref{the:black-box-complexity-onemax}, we assume without
  loss of generality that the search point $1^n$ is the optimum, and
  let $(X_t)_{t\in\mathbb{N}}$ be the potential as defined before.

  We apply the last part of Theorem~\ref{theorem:general-drift-theorem} (iv),
  with the parameters $g(x):= x$, $\xmin:=1$, $\xmax:=n$, $a:=0,$
  $S:=\{0\}\cup[\xmin,\xmax]$, and $\beta_l(t):=8\lambda$, for all
  $t\in\mathbb{N}$. We consider the number of \emph{parallel runs}
  $T'$ until the process reaches potential $a=0$.

  Define $c:=\frac{3}{10}\gamma$ where
  $\gamma:=\ln\left(\frac{3}{4}\sqrt{2}\right)$.
  By Lemma~\ref{lemma:mgf-bound}
  \begin{align*}
    \expect{e^{\gamma (g(X_t)-g(X_{t+1}))}\filtcond{ X_t > a}}
    \leq\;& \expect{e^{\gamma\Delta^{(\lambda)}(s)}}
    \leq 8\lambda = \beta_\ell(t)
  \end{align*}
  Furthermore, by the definition of the process, the set $S\cap\{x\mid
  x\leq a\}=\{0\}$ is absorbing, thus for $t:=\frac{cn}{\ln^+\lambda}$,
  \begin{align*}
    \prob{T'<t\mid X_0>0}
    &\leq  \left(\prod_{i=0}^{t-1} \beta_{\mathrm{\ell}}(i)\right)\cdot e^{-\gamma (g(X_0)-g(a))}\\
    &< (8\lambda)^t\cdot e^{-\gamma n}\\
    & =    (8\lambda)^{\frac{c n}{\ln^+\lambda}}\cdot e^{-\gamma n}\\
    & =    e^{\left(\frac{c n}{\ln^+\lambda}\right)\ln(8\lambda)-\gamma n}
    \intertext{using that $\ln(8\lambda)=\ln(\lambda)+3\ln(2)\leq 3\ln^+\lambda$ gives}
    & \leq    e^{(3c-\gamma)n}\\
    & =    e^{-\gamma n/10}.
  \end{align*}
  The result follows by taking into account that the algorithm makes
  $\lambda$ fitness evaluations per iteration, \ie, $T=\lambda T',$
  and that $c>1/60.$
\end{proof}

\subsection{Lower Bound \boldmath$\Omega(n \log n)$ with overwhelming probability}

Now we show a lower bound of $\Omega(n \log n)$ with overwhelming probability. Note that this result is independent of $\lambda$ and thus unrelated to parallel black-box complexity; it gives limitations for general (parallel or non-parallel) unary unbiased black-box algorithms. Recall that every $\lambda$-parallel unary unbiased algorithm is also a unary unbiased algorithm, hence the result applies to a strictly larger class of algorithms. Previously only lower bounds on the expectation were known: \citet{Lehre2012} showed an asymptotic bound of $\Omega(n \log n)$ and~\citet{Doerr2016a} presented a more precise lower bound of $n \ln n - O(n)$.
\begin{theorem}
\label{the:n-log-n-bound}
For every unary unbiased black-box algorithm $\mathcal{A}$ and every constant $0 < \delta \le 1$, the probability that $\mathcal{A}$ finds any fixed target search point $x^*$ within $(1-\delta)n \ln n$ steps is $\exp(-\Omega(n^{\delta}/\log n))$.
\end{theorem}
Before presenting the proof of Theorem~\ref{the:n-log-n-bound}, we present the main idea behind the proof, and the challenges to overcome.

The proof will be based on the following well-known ``coupon collector'' argument that we discuss first for a simple algorithm such as Randomised Local Search (RLS) or the (1+1)~EA. For these algorithms, we can argue that with high probability there will be $cn$ bits in the initial search point that differ from the optimum, for an appropriate constant $0 < c < 1/2$.
Each such bit has a probability of $1/n$ of being flipped in each step of the algorithm. For a time period of $T:= (1-\delta) (n-1) \ln n$ steps, the probability that any fixed bit is never being flipped is at least
\[
\left(1-\frac{1}{n}\right)^T \ge \left(1-\frac{1}{n}\right)^{(1-\delta)(n-1) \ln n} \ge
n^{-(1-\delta)}
\]
using $(1-1/n)^{n-1} \ge 1/e$. Now the probability that there is a bit among the $cn$ incorrect bits that is never being flipped is at least
\[
\left(1 - n^{-(1-\delta)}\right)^{cn} \le \exp(-cn^{\delta}).
\]
This implies that with the above probability the optimum has not been found in $T = \Omega(n \log n)$ steps.

This argument works for RLS and the (1+1)~EA for the following reasons:
\begin{enumerate}
\item The algorithms evolve a single lineage from the initial search point, which allows us to argue with ``incorrect'' bits that need to be flipped at least once.
\item The same variation operator is applied at all times, which establishes the formula $(1-1/n)^T$.
\item All bits are treated independently, which is implicitly used in the derivation of the term $(1-n^{-(1-\delta)})^{cn}$.
\end{enumerate}
In order to prove Theorem~\ref{the:n-log-n-bound}, we have to consider \emph{all} unary unbiased black-box algorithms, for which the above properties do not hold. In particular, algorithms may easily generate several lineages. This makes it unclear how ``incorrect'' bits can be defined.
Also note that an algorithm might flip many ``incorrect'' bits in one step simply by choosing a very large radius. So the simple argument that we need to flip all incorrect bits at least once breaks down.
Algorithms may choose different variation operators at different times, possibly depending on fitness values generated so far. This makes it difficult to argue that no variation flips a bit over a period of time. Finally, mutations with a fixed radius $r \ge 2$ may introduce dependencies between bits, which needs to be addressed.


We tackle these challenges as follows.
Assume w.\,l.\,o.\,g.\ that $x^* = 1^n$.
We give away knowledge of all search points $x$ that have Hamming distance at least $n^* := n/(2^{13}\ln n)$ to both $0^n$ and $1^n$. Hence we start with a potential of $s = n^*$. Moreover, whenever the algorithm decreases the potential from $s$ to $s' < s$, we grant the algorithm knowledge of all solutions with Hamming distance at least $s'$ from both $0^n$ and $1^n$. This assumption implies that the current knowledge of the algorithm can be fully described by the current potential, and the progress of the algorithm can be bounded by considering the transitions of the potential.

Note that all solutions with the same potential are isomorphic to the algorithm. Pick a set of $n^*$ bit positions, w.\,l.\,o.\,g.\ the first $n^*$ ones. We define these bits as ``incorrect'' bits that need to be set to 1 in order to reach the optimum. Since the behaviour of the algorithm is fully determined by the current potential, and the bit positions are irrelevant for transitions between potential values, we may assume w.\,l.\,o.\,g.\ that the algorithm, whenever performing a variation of a search point $x_t$ with $\ones{x_t}$ ones, it always picks $x_t$ from the set of all search points with $\ones{x_t}$ ones such that $\min\{n-\ones{x_t}, n^*\}$ ``incorrect'' bit positions have value~0. Such a search point always exists as otherwise the potential would be less than $\ones{x_t}$ at time~$t$, which is a contradiction.

Now variations that decrease the potential by decreasing the number of zeros will fix some of the incorrect bits accordingly. Variations that do not decrease the potential only create search points that are already known and thus can be ignored as they have no effect. Hence we require that these incorrect bits are flipped \emph{in variations that decrease the potential}.

Having laid the foundation for arguing with ``incorrect'' bits being fixed, we now show that with overwhelming probability, $\mathcal{A}$ does not find $1^n$ within $T := (1-\delta)(n-1) \ln n$ steps.

Note that $\mathcal{A}$ can choose the radius in each step. We distinguish between single-bit variations where $r=1$ (or, symmetrically, $r=n-1$) and multi-bit variations where $2 \le r \le n-2$. We first show that in at most~$T$ steps with multi-bit variations, not too many incorrect bits are being fixed. Then we show later that at most $T$ single-bit variations are not enough to fix all incorrect bits that are not being fixed by multi-bit variations. Note that the algorithm can interleave single-bit variations and multi-bit variations arbitrarily. Our arguments work for arbitrary sequences of single-bit and multi-bit variations; they even hold if the algorithm is allowed to make $T$ single-bit variations \emph{and} $T$ multi-bit variations at the cost of $T$ queries.

The following lemma considers multi-bit variations and bounds transition probabilities of the potential.
\begin{lemma}
\label{lem:bounding-probability-of-progress}
Let $s \le n^*$ for $n^* := n/(2^{13}\ln n)$, then for every $m \in [s, 2n^*] \cup [n-2n^*, n-s]$, every radius $2 \le r \le n-2$ and every $1 \le z \le n$ we have
\[
\Prob{\Delta_0(s, m, r) = z} \le \left(\frac{16n^*}{n}\right)^2 \cdot 2^{-z}.
\]
If $2n^* < m < n-2n^*$ we have
\[
\Prob{\Delta_0(s, m, r) = z} \le e^{-\Omega({n^*}^2/n)}.
\]
\end{lemma}
\begin{proof}
Recall that by~\eqref{eq:Delta-symmetry} it
%
suffices to consider the case $m \le n/2$.
If $2n^* \le m \le n/2$ then by Lemma~\ref{lem:applying-Chvatal}
\[
 \Prob{\Delta_0(s,m,r) > 0}
\leq \exp\left(-\frac{(m-s)^2}{2r}\right) = e^{-\Omega({n^*}^2/n)}.
\]

Now assume $s \le m \le 2n^*$.
As shown in the proof of Lemma~\ref{lem:progress-m},
\[
\Prob{\Delta_0(s, m, r) = z} \le \left(\frac{4m}{n}\right)^{(z+r+m-s)/2}
\le \left(\frac{8n^*}{n}\right)^{(z+r)/2}
\]
We claim that the above is bounded by $\left(\frac{4n^*}{n}\right)^2 \cdot 2^{-z}$ for all $z \ge 1$ and $r \ge 2$.

Note that for $z=1$ and $r=2$ we have $\Prob{\Delta_0(s, m, r) = z} = 0$ as the progress must be an even number.
For $z=1$ and $r \ge 3$ we get
\[
\left(\frac{8n^*}{n}\right)^{(z+r)/2} \!\! = \left(\frac{8n^*}{n}\right)^2 \cdot \left(\frac{8n^*}{n}\right)^{(r-3)/2}
 \le \left(\frac{16n^*}{n}\right)^2 \cdot 2^{-1}.
\]

For $z=2$ we get
\[
\left(\frac{8n^*}{n}\right)^{(z+r)/2} \!\! = \left(\frac{8n^*}{n}\right)^2 \cdot \left(\frac{8n^*}{n}\right)^{(r-2)/2}
 \le \left(\frac{16n^*}{n}\right)^2 \cdot 2^{-2}.
\]

For $z \ge 3$ we have, using $(8n^*/n)^{1/2} \le 1/2$,
\[
\left(\frac{8n^*}{n}\right)^{(z+r)/2} \!\! \le \left(\frac{8n^*}{n}\right)^2 \cdot \left(\frac{8n^*}{n}\right)^{z/2}
\le \left(\frac{8n^*}{n}\right)^2 \cdot 2^{-z}. \qedhere
\]
\end{proof}

Using Lemma~\ref{lem:bounding-probability-of-progress} now allows us to express the progress of any algorithm using stochastic domination and a combination of two simple random variables:
\begin{lemma}
\label{lem:dominating-distribution}
Let $s \le n^*$ for $n^* := n/(2^{13}\ln n)$, then for every $s \le m \le n-s$ and every radius $2 \le r \le n-2$ the progress $\Delta(s, m, r)$ is stochastically dominated by
\[
2X_t Y_t
\]
where $X_t \in \{0, 1\}$ is a Bernoulli variable with $\Prob{X_t = 1} = \left(\frac{16n^*}{n}\right)^2$ and $Y_t$ is a geometric random variable with parameter $1/2$, $X_t$ and $Y_t$ being independent of each other and independent of other time steps $t' \neq t$.
\end{lemma}
\begin{proof}
By Lemma~\ref{lem:bounding-probability-of-progress} and the definition of $X_t, Y_t$,
\[
\Prob{\Delta_0(s, m, r) = z} \le \left(\frac{16n^*}{n}\right)^2 \cdot 2^{-z} = \Prob{X_t Y_t = z}
\]
for every~$z \ge 1$ and all $m \in [s, 2n^*] \cup [n-2n^*, n-s]$. The same clearly also holds in case $2n^* < m < n-2n^*$ by the second statement of Lemma~\ref{lem:bounding-probability-of-progress}.

The probability bounds for $\Delta_0$ also apply to $\Delta_1$ by symmetry of zeros and ones, and thus by the union bound $\Prob{\Delta(s, m, r) \ge z} \le \Prob{\Delta_0(s, m, r) \ge z} + \Prob{\Delta_1(s, m, r) \ge z}$ we get
$
\Prob{\Delta(s, m, r) \ge z} \le 2 \cdot \Prob{X_t Y_t \ge z}
$.
\end{proof}

We use Lemma~\ref{lem:dominating-distribution} to show tail bounds for the progress made in multi-bit variations. The following lemma shows that at most half of the incorrect bits are being fixed by multi-bit variation steps, even when considering a time span of $n \ln n$ steps instead of $(1-\delta)n \ln n$.
\begin{lemma}
\label{lem:tail-bound-for-multi-bit-variations}
Let $n^* := n/(2^{13}\ln n)$. Within $T := n \ln n$ multi-bit variation steps at most $n^*/2$ incorrect bits are being fixed, with probability
$1-2^{-\Omega(n/\log n)}$.
\end{lemma}
\begin{proof}
We give a tail bound for the sum of variables $X_t Y_t$ defined in Lemma~\ref{lem:dominating-distribution}; by stochastic domination, the tail bound then also holds for the real progress. Recall that $X_t$ as well as $Y_t$ are both sequences of iid variables and that all variables are mutually independent.

By Chernoff bounds, with overwhelming probability the number of $X_t$ variables attaining value~1 is bounded by at most twice its expectation:
\begin{align*}
\Prob{\sum_{t=1}^T X_t \ge 2T \left(\frac{16n^*}{n}\right)^2} \le\;& \exp\left(-\frac{T}{3} \left(\frac{16n^*}{n}\right)^2\right)
= e^{-\Omega(n/\log n)}.
\end{align*}
If $\sum_{t=1}^T X_t \le \left\lfloor 2T \left(\frac{16n^*}{n}\right)^2 \right\rfloor =: k$ then there are at most $k$ variables $Y_t$ that contribute to $\sum_{t=1}^T X_t Y_t$. For ease of notation, we assume that these are variables $Y_1, \dots, Y_k$.

We apply Chernoff bounds for sums of geometric random variables~\cite[Theorem~3]{Doerr2011d} to bound the contribution of $k$ variables $Y_1, \dots, Y_k$. Note that ${\E{\sum_{t=1}^k Y_t} = 2k}$.
\begin{align*}
\Prob{\sum_{t=1}^k Y_t \ge 4k} \le\;& \exp\left(-\frac{k-1}{4}\right)
= e^{-\Omega(n/\log n)}.
\end{align*}
Hence if both ``typical'' events occur,
\begin{align*}
\sum_{t=1}^T 2X_t Y_t \le\; 8k
\le\;& 16T \cdot \frac{16^2 n^*}{n^2} \cdot n^*
= 16n \ln(n) \cdot \frac{2^{-5}}{n \ln n} \cdot n^*
=\; n^*/2.
\end{align*}

Taking the union bound for the two probabilities $2^{-\Omega(n/\log n)}$ that the typical events do not happen completes the proof.
\end{proof}

Now we are ready to give a proof for Theorem~\ref{the:n-log-n-bound}.
\begin{proof}[Proof of Theorem~\ref{the:n-log-n-bound}]
As explained earlier, it suffices to consider $n^*$ incorrect bits and to show that with the claimed probability not all of these bits will be fixed within $T$ unbiased variations.

Lemma~\ref{lem:tail-bound-for-multi-bit-variations} implies that with overwhelming probability there exist $n^*/2$ incorrect bits that are not being fixed by up to~$T$ multi-bit variations. We now use coupon collector argument (similar to those sketched earlier) to show that, in up to $T$ single-bit variations, with overwhelming probability these $n^*/2$ incorrect bits will not all be fixed.


The probability that any fixed bit~$i$ will not be flipped in a single-bit variation amongst the first $T$ steps is at least, using $(1-1/x)^{x-1} \ge 1/e$ for $x > 1$,
\[
\left(1 - \frac{1}{n}\right)^T
= \left(1 - \frac{1}{n}\right)^{(1-\delta)(n-1) \ln n}
\ge n^{-(1-\delta)}.
\]
Hence the probability that a fixed bit~$i$ will be flipped in at to $T$ single-bit variations is at least $1-n^{-(1-\delta)}$. Hence the probability that all of the $n^*/2$ incorrect bits are being flipped in $T$ steps is at most
\[
(1-n^{-(1-\delta)})^{n^*/2} \le \exp(-\Omega(n^{\delta}/\log n)). \qedhere
\]
\end{proof}

Theorems~\ref{the:tail-bound-lambda-term} and~\ref{the:n-log-n-bound} imply our main result, Theorem~\ref{the:main-result-tail-bounds}.
\begin{proof}[Proof of Theorem~\ref{the:main-result-tail-bounds}]
Fix a target search point $x^*$ from the target set.
By Theorem~\ref{the:tail-bound-lambda-term} the probability of finding $x^*$ within $\frac{\lambda n}{60\ln^+\lambda}$ steps is $\exp(-\Omega(n))$.
Applying Theorem~\ref{the:n-log-n-bound} with parameter $\delta$ yields that the probability of finding $x^*$ within $(1-\delta) n \ln n$ steps is $\exp(-\Omega(n^{\delta}/\log n))$.
By the union bound, the probability that one of these lower bounds does not apply is $\exp(-\Omega(n)) + \exp(-\Omega(n^{\delta}/\log n)) \le 2\exp(-\Omega(n^{\delta}/\log n))$. Repeating the above arguments for all target search points and using a union bound over
at most $\exp(o(n^{\delta}/\log n))$ search points yields an overall probability bound of
\begin{align*}
& \exp(o(n^{\delta}/\log n)) \cdot 2\exp(-\Omega(n^{\delta}/\log n)) \\
=\;& \exp(-\Omega(n^{\delta}/\log n)+ o(n^{\delta}/\log n) + \ln 2)\\
=\;& \exp(-\Omega(n^{\delta}/\log n)).
\end{align*}
Finally, the claimed equality
\[
\max\left\{\frac{\lambda n}{60\ln^+\lambda}, (1-\delta) n \ln n\right\}
= \mathord{\Omega}\mathord{\left(\frac{\lambda n}{\ln^+\lambda} + n \ln n\right)}
\]
follows from $\max\{x, y\} \ge (x+y)/2$ and $1-\delta = \Omega(1)$.
\end{proof}

\section{Black-Box Complexity Results for Illustrative Function Classes}
\label{sec:applications-of-tail-bounds}

In this section we give a number of examples of how to exploit the fact that our lower bounds apply to the time for finding an arbitrary target set of up to exponentially many search points. This leads to novel results for functions with many global optima, but can also be used to bound the time for reaching local optima or search points within a certain distance from any local or global optimum.

\subsection{Black-Box Complexity Lower Bounds for Functions with Many Optima}

Previous black-box complexity results like Theorem~\ref{the:black-box-complexity-onemax} or results on (non-parallel) unbiased black-box complexity~\cite{Lehre2012} were limited to functions with a unique optimum. These results apply to popular test functions like \onemax and \LO and function classes like linear functions or monotone functions~\cite{monotone-journal}. However, they do not apply when considering functions with more than one optimum.
Apart from tailored analyses for specific problems classes (e.\,g.\ problems from combinatorial optimisation~\citep{Doerr2013b}), we are not aware of any generic black-box complexity results that apply to functions with multiple optima.

Theorem~\ref{the:main-result-tail-bounds} overcomes this limitation, yielding novel black-box complexity results for the unary unbiased black-box complexity and its $\lambda$-parallel variant across a range of problems with several global optima, including some widely studied problem classes. These black-box complexity results give general limitations that can serve as baselines for performance comparisons and guide the search for the most efficient algorithms, including those using parallelism most effectively (as demonstrated successfully for \onemax in Section~\ref{sec:optimal-algorithm-onemax}).

There are many examples of relevant problem classes to which Theorem~\ref{the:main-result-tail-bounds} applies. The most obvious class is that of all functions with $\exp(o(n^{\delta}/\log n))$ optima. Note that when choosing, say, $\delta := 0.995$ then $\exp(n^{0.99}) \le \exp(o(n^{\delta}/\log n))$; the reader may choose to think of the latter expression as $\exp(n^{0.99})$ as this may be easier to digest.

Following~\citet{Witt2006}, the mentioned function class includes problems where all optima have at most $n^\delta/\log^3 n$ ones or at most $n^\delta/\log^3 n$ zeros.
This is because the number of such search points is bounded by
\begin{equation}
\label{eq:area-of-hypercube}
2\sum_{i=0}^{n^\delta/\log^3 n} \binom{n}{i} = \mathord{O}\mathord{\left(n^{n^\delta/\log^3 n}\right)}
= \exp(o(n^{\delta}/\log n)),
\end{equation}
where the last step used $n^{n^{\delta/\log^3 n}} = \exp(\Theta(n^{\delta/\log^2 n})) = \exp(o(n^{\delta/\log n}))$.

In the following we survey a number of illustrative problems that have been studied previously and for which we give the first black-box complexity results.
In terms of combinatorial problems, there are a lot of well-studied problems with a property called \emph{bit-flip symmetry}: flipping all bits gives a solution of the same fitness. This means that there are always at least two global optima. Such problems have been popular as search algorithms need to break the symmetry between good solutions~\citep{Goldberg2002}.

Well-known examples include the function $\twomax := \max\{\sum_{i=1}^n x_i, {\sum_{i=1}^n (1-x_i)}\}$ \citep{Goldberg2002}, which has been used as a challenging test bed in theoretical studies of diversity-preserving mechanisms~\citep{Oliveto2018,Covantes2018a,CovantesOsuna2018}.
The function \textsc{H-Iff} (Hierarchical If and only If)~\citep{Watson1998} consists of hierarchical building blocks that need to attain equal values in order to contribute to the fitness. It was studied theoretically~\citep{Dietzfelbinger2003,Goldman2016} and is frequently used in empirical studies, see, e.\,g.~\citep{Thierens2013,Goldman2015}.

In terms of classical combinatorial problems, the \textsc{Vertex Colouring} problem asks for an assignment of colours to vertices such that no two adjacent vertices share the same colour. For two colours, a natural setting is to use a binary encoding for the colours of all vertices and to maximise the number of bichromatic edges (edges with differently coloured end points). A closely related setting is that of simple Ising models, where the goal is to \emph{minimise} the number of bichromatic edges. For bipartite (that is, 2-colourable) graphs, this is identical to maximising the number of bichromatic edges as inverting one set of the bipartition turns all monochromatic edges into bichromatic ones and vice versa.
Previous theoretical work includes evolutionary algorithms on ring/cycle graphs~\citep{Fischer2005}, the Metropolis algorithm on toroids~\citep{Fischer2004} and evolutionary algorithms on binary trees~\citep{Sudholt2005}.

Other combinatorial problems with bit-flip symmetry include cutting and selection problems. Given an undirected graph, the problems \textsc{MaxCut} and \textsc{MinCut} seek to partition the graph into two non-empty sets such as to maximise or minimise the number of edges running between those two sets, respectively. Using a straightforward binary encoding for all vertices, this results in bit-flip symmetry and multiple optima. Theoretical studies of evolutionary algorithms on cutting problems include~\citet{Neumann2011} and~\citet{Sudholt2010}; the latter paper considers a simple instance of two equal-sized cliques that leads to two complementary optima.
Concerning selection problems, the well-known NP hard \textsc{Partition} problem asks whether it is possible to schedule a set of $n$ jobs on two identical machines such that both machines will have identical loads. An optimisation problem is obtained by trying to minimise the load of the fuller machine, also called the \emph{makespan}. A straightforward encoding is used: every bit indicates which machine the corresponding job should be assigned to.
\Citet{Witt2005} analysed the performance of the \EA for this problem, including random instance models where job sizes are drawn randomly from a real range, according to a uniform or an exponential distribution, respectively. In both cases such instances will almost surely have two complementary optima\footnote{More than two optima only exist if there are different combinations of job sizes (beyond symmetries) that add up to the same value. Since the weight of each job size is drawn from a continuous range and the number of values that could lead to equal values is finite, this almost surely never happens.}.

\Citet*{Wegener2005c} considered monotone polynomials: a sum of monomials (products of variables, e.\,g.\ $x_1 x_3 x_4$) with positive weights. Here $1^n$ is always a global optimum, but more optima can exist if there are variables that do not appear in any monomial: each such variable doubles the number of optima as it is not relevant for the fitness. Hence if there are $o(n^\delta/\log n)$ such variables then there are at most $2^{o(n^\delta/\log n)} \le \exp(o(n^\delta/\log n))$ optima.

\Citet*{Jansen2016} presented instance classes called \emph{nearest peak functions} and \emph{weighted nearest peak functions}. Both are defined with respect to an arbitrary number of peaks: search points with an associated height and slope. For nearest peak functions the fitness of a search point is determined by its closest peak: for the peak itself the fitness is equal to the height of the peak and for other search points the fitness decreases gradually with the distance from the peak, according to the slope of the peak. Weighted nearest peak functions are defined similarly, but all peaks are considered and higher peaks can dominate shallower peaks. This function class was introduced as a test bed allowing to create an arbitrary number of optima. It is shown in~\citet{Jansen2016} that the set of local optima is a subset of all peaks. Hence the number of peaks is an upper bound on the number of global (and local) optima. The two function classes were named Jansen-Zarges function classes in~\citet{Covantes2018a}, where they were used as benchmarks for the \emph{clearing} diversity mechanism.

Finally we consider random planted \textsc{Max-3-Sat} instances as a popular benchmark model in both experimental~\cite{Goldman2014} and theoretical studies~\cite{Sutton2014,Doerr2015a,Buzdalov2017}. The fitness function is the number of satisfied clauses and each clause contains exactly 3 literals
(negated or non-negated variables from the set $\{x_1, \dots, x_n\}$).
In this model, we fix a planted optimum $x^*$ and generate clauses independently such that they are satisfied by~$x^*$. This means that at least one literal needs to evaluate to \texttt{true} in~$x^*$. The variables for each clause are chosen uniformly at random (with or without replacement) from $\{x_1, \dots, x_n\}$. We may assume that instances are generated by first deciding which of the 3 literals will match~$x^*$ and which won't. In a second step, the indices of variables will be picked. We further assume that there is at least a constant probability $c_1$ of a clause having one matching literal and at least a constant probability $c_3$ of a clause having three matching variables\footnote{This is the case in~\cite{Sutton2014,Doerr2015a,Buzdalov2017} where implicitly $c_1 = 3/7$ and $c_3 = 1/7$ and in~\cite{Goldman2014} where $c_1=4/6$ and $c_3=1/6$. The latter probabilities favour clauses with only one matching literal in order not to give an obvious bias towards the values of~$x^*$. Note that we do not care about the value of $c_2$ (two matching literals).}.
In this setup, $x^*$ is a global optimum, but there may be more global optima. We argue that the number of optima is bounded if the number of clauses, $m$, is chosen large enough.

Consider a solution $x$ with Hamming distance $H := H(x, x^*)$ to~$x^*$.
We argue that for any clause, the probability that the clause will be satisfied under~$x$ is $\Omega(H/n)$. If $H \le n/2$ then with probability $c_1$ we will choose one matching literal and the probability that only the variable of this literal will be chosen among the $H$ ones that differ in $x$ and $x^*$ is $\Omega(H(n-H)^2/n^3) = \Omega(H/n)$. Likewise, if $H > n/2$ then with probability $c_3$ we will choose three matching literals and the probability that they are all different in $x$ and $x^*$ is $\Omega(H^3/n^3) = \Omega(H/n)$.
Now since all clauses are generated independently, the probability that all $m$ clauses are satisfied under~$x$ is $(1-\Omega(H/n))^m \le \exp(-\Omega(Hm/n))$.

Hence for all search points~$x$ with $H \ge n^\delta/\log^3 n$ the probability that $x$ is a global optimum is at most $\exp(-\Omega(n^{\delta}/(\log^3 n) \cdot m/n)) = \exp(-\Omega(n \log n))$ if the number of clauses is $m = \Omega(n^{2-\delta}\log^4 n)$. In this case, the probability that any such search point will be a global optimum is at most $2^n \cdot \exp(-\Omega(n \log n)) = \exp(-\Omega(n \log n))$, a failure probability so small that it can be absorbed in the failure probabilities for our tail bounds.
Now, with overwhelming probability the number of global optima is bounded by the number of search points with Hamming distance less than $n^\delta/\log^3 n$ from~$x^*$. By~\eqref{eq:area-of-hypercube}, this number is $\exp(o(n^{\delta}/\log n))$.

The following theorem summarises all the above.
\begin{theorem}
\label{the:main-result-tail-bounds-applications}
Every unary unbiased $\lambda$-parallel black-box algorithm $\mathcal{A}$
needs more than
\[
\max\left\{\frac{\lambda n}{60\ln^+\lambda}, (1-\delta) n \ln n\right\}
= \mathord{\Omega}\mathord{\left(\frac{\lambda n}{\ln^+\lambda} + n \ln n\right)}
\]
evaluations, with probability $1-\exp(-\Omega(n^{\delta}/\log n))$, to find a global optimum for all of the following settings.
\begin{enumerate}
\item All functions with $\exp(o(n^{\delta}/\log n))$ optima.
\item All functions where all optima have at most $n^\delta/\log^3 n$ ones or at most $n^\delta/\log^3 n$ zeros.
\item $\twomax := \max\{\sum_{i=1}^n x_i, \sum_{i=1}^n (1-x_i)\}$.
\item \textsc{H-Iff} (Hierarchical If and only If).
\item Vertex colouring/Ising model problems: maximising or minimising the number of bichromatic edges when trying to colour a connected bipartite graph with 2 colours.
\item \textsc{MinCut} instances with two equal-sized cliques.
\item \textsc{Partition} instances having two symmetric optimal solutions (which almost surely applies to random instances)
\item Monotone polynomials with positive weights where all but $o(n^\delta/\log n)$ variables appear in at least one monomial.
\item Jansen-Zarges nearest peak functions and weighted nearest peak functions with $\exp(o(n^\delta/\log n))$ peaks.
\item Random planted \textsc{Max-3-Sat} instances as described above with at least $m=\Omega(n^{2-\delta}\log^4 n)$ clauses.
\end{enumerate}
The expected time also satisfies the asymptotic bound.
\end{theorem}

%
%

\subsection{Lower Bounds on the Time to Reach Local Optima}

For many multimodal problems evolutionary algorithms are likely to need a much larger time than indicated by the lower bounds from Theorem~\ref{the:main-result-tail-bounds-applications}. When put in perspective, our bounds may appear to be quite loose for some of the harder problems considered. However, our lower bounds can also be applied to bound the time until any unary unbiased black-box algorithm has found a local optimum, or any search point of reasonably high fitness, if the number of such points is bounded.

Example applications include functions with $\exp(o(n^{\delta}/\log n))$ \emph{local} optima, including those where all local optima have at most $n^\delta/\log^3 n$ ones or at most $n^\delta/\log^3 n$ zeros. The latter function class includes the well-known $\textsc{Jump}_k$ functions~\citep{Droste2002,Dang2017}, where a gap of Hamming distance~$k$ has to be
``jumped'' to reach a global optimum,
with parameter $k \le n^\delta/\log^3 n$: here all search points with $k$ zeros are local optima, in addition to the global optimum~$1^n$. A similar function class $\textsc{Cliff}_d$ was used in~\cite{Jagerskupper2007a,Paixao2016,Corus2017}, where the same holds for $d$ in lieu of~$k$; the difference between these two functions is that in the region ``between'' local and global optima $\textsc{Jump}_k$ has a gradient pointing back towards the local optima whereas $\textsc{Cliff}_d$ points towards the global optimum~$1^n$.

Functions with difficult local optima include a modified version of $\twomax$ used in~\citep{Friedrich2009}: in $\twomax' := \max\{\sum_{i=1}^n x_i, \sum_{i=1}^n (1-x_i)\} + \prod_{i=1}^n x_i$ the point $1^n$ is the only global optimum and $0^n$ is a local optimum it is very hard to escape from.
A combinatorial example of a \textsc{MaxSat} instance with difficult local optima was studied in the context of evolutionary algorithms in~\citet{Droste2002a}, with variables $x_1, \dots, x_n$ and clauses
\begin{equation}
\label{eq:maxsat-instance}
\{(x_i \vee \overline{x_j} \vee \overline{x_k}) \mid i \neq j \neq k \neq i\} \cup \{(x_i) \mid 1 \le i \le n\}.
\end{equation}
Here the optimum is again $1^n$, and all $n$ search points with a single 1-bit are local optima. Likewise, the \textsc{MinCut} instance from Theorem~\ref{the:main-result-tail-bounds-applications} has $O(n)$ local optima as well: all search points with exactly one 1-bit or one 0-bit are locally optimal. \Citet{Sudholt2010} further presents a hard \textsc{Knapsack} instance with $(n+1)/2$ ``small'' objects of weight and value $n$ and $(n-1)/2$ ``big'' objects of weight and value $n+1$. The weight limit is set to $(n+1)/2 \cdot n$, such that including all small objects yields a global optimum, but selecting all but one big object gives a local optimum. Similar as above, the number of local optima is $O(n)$.

Finally, the arguments for Jansen-Zarges function classes also hold with respect to the number of local optima.

The following theorem summarises all the above.
\begin{theorem}
\label{the:main-local-optima-applications}
Every unary unbiased $\lambda$-parallel black-box algorithm $\mathcal{A}$
needs more than
\[
\max\left\{\frac{\lambda n}{60\ln^+\lambda}, (1-\delta) n \ln n\right\}
= \mathord{\Omega}\mathord{\left(\frac{\lambda n}{\ln^+\lambda} + n \ln n\right)}
\]
evaluations, with probability $1-\exp(-\Omega(n^{\delta}/\log n))$, to find a \textbf{local} or global optimum for all of the following settings.
\begin{enumerate}
\item All functions with $\exp(o(n^{\delta}/\log n))$ local optima.
\item All functions where all local optima have at most $n^\delta/\log^3 n$ ones or at most $n^\delta/\log^3 n$ zeros.
\item $\textsc{Jump}_k$ functions with $k \le n^\delta/\log^3 n$.
\item $\textsc{Cliff}_d$ functions with $d \le n^\delta/\log^3 n$.
\item $\twomax := \max\{\sum_{i=1}^n x_i, \sum_{i=1}^n (1-x_i)\}$ as well as the modified \twomax function $\twomax' := \max\{\sum_{i=1}^n x_i, \sum_{i=1}^n (1-x_i)\} + \prod_{i=1}^n x_i$
\item \textsc{MinCut} instances with two equal-sized cliques.
\item The hard \textsc{MaxSat} instance from~\eqref{eq:maxsat-instance}.
\item The hard \textsc{Knapsack} instance mentioned above.
\item Jansen-Zarges nearest peak functions and weighted nearest peak functions with $\exp(o(n^\delta/\log n))$ peaks.
\end{enumerate}
The expected time also satisfies the asymptotic bound.
\end{theorem}


We can even push our applications a bit further. Again using~\eqref{eq:area-of-hypercube}, there are at most $\exp(o(n^{\delta}/\log n))$ search points within a Hamming ball of radius $n^\delta/\log^3 n$ around any search point. If there are $\exp(o(n^{\delta}/\log n))$ global or local optima then the number of all search points within the union of Hamming balls around all these points is still
$\exp(o(n^{\delta}/\log n)) \cdot \exp(o(n^{\delta}/\log n)) = \exp(o(n^{\delta}/\log n))$. Hence our main result from Theorem~\ref{the:main-result-tail-bounds} still applies when considering the time to get to within Hamming distance $n^\delta/\log^3 n$ of any global or local optimum.
\begin{theorem}
Theorem~\ref{the:main-result-tail-bounds-applications} and Theorem~\ref{the:main-local-optima-applications} still apply when replacing ``to find a global optimum'' with ``to find any search point within Hamming distance $n^\delta/\log^3 n$ to any global optimum'' in
Theorem~\ref{the:main-result-tail-bounds-applications} and
replacing ``to find a local or global optimum'' with
``to find any search point within Hamming distance $n^\delta/\log^3 n$ to any local or global optimum'' in Theorem~\ref{the:main-local-optima-applications}.
\end{theorem}

In particular, this implies that with overwhelming probability no unary unbiased black-box algorithm can find a search point of fitness at least $n-n^\delta/\log^3 n$ for \onemax, \LO and \twomax within the stated time.
In other words, the expected fitness after the stated time is $n-n^\delta/\log^3 n + o(1)$ (where the $o(1)$ term accounts for an exponentially small failure probability, in case of which the fitness could be as large as~$n$). Such results are known as \emph{fixed-budget results}~\citep{Jansen2014,Doerr2013c}. This shows that our $\lambda$\nobreakdash-parallel black-box complexity results with tail bounds can be applied in a large variety of settings.

%
%
%

\section{Conclusions and Future Work}


We have introduced the parallel unbiased black-box complexity to quantify the limits on the performance of parallel search heuristics, including offspring populations, island models and multi-start methods. We proved that \emph{every} $\lambda$-parallel unbiased black-box algorithm needs at least $\Omega(\lambda n/\log^+(\lambda) + n \log n)$ function evaluations on every function with unique optimum, and at least $\Omega(\lambda n/(\log^+(\lambda/n)) + n^2)$ function evaluations on \LO. Corresponding parallel times are by a factor of~$\lambda$ smaller.
For \LO and \OM we identified the cut-off point for $\lambda$, above which the asymptotic number of function evaluations increases, compared to non-parallel algorithms ($\lambda=1$). All smaller $\lambda$ allow for linear speedups with regard to the parallel time. For \OM this cut-off point is higher than that for the standard \lEA; optimal performance for all~$\lambda$ is achieved by a \lEA with an adaptive mutation rate.

In a novel and more detailed analysis we have established tail bounds showing that the lower bound $\mathord{\Omega}\mathord{\left(\frac{\lambda n}{\ln^+\lambda} + n \ln n\right)}$ holds with overwhelming probability, for parallel and non-parallel algorithms (where $\lambda=1$) and for finding any target set of search points we can choose. This makes it a very general, powerful and versatile statement: we obtain lower bounds on the optimisation time on functions with many optima, the time to find a local optimum, and the time to even get close to any local or global optimum. We demonstrated the usefulness of this approach by deriving the first black-box complexity lower bounds for a range of popular and illustrative problems, from
synthetic problems (\twomax, \textsc{H-IFF}, \textsc{Jump}$_k$, \textsc{Cliff}) to classes of multimodal benchmark functions~\cite{Jansen2016} and important problems from combinatorial optimisation such as \textsc{Vertex Colouring}, \textsc{MinCut}, \textsc{Partition}, \textsc{Knapsack} and \textsc{MaxSat}.

A major open problem for future work is to derive lower bounds for the $\lambda$-parallel unbiased black-box complexity when allowing binary operators like crossover, or operators combining many search points as in EDAs or swarm intelligence algorithms. Currently even in the non-parallel case no non-trivial lower bounds on the binary unbiased black-box complexity are known.

%


\bibliographystyle{apalike}
\bibliography{structured-populations}

\appendix
\section{Drift Theorems}
\begin{theorem}[General Drift Theorem~\citep{LehreWitt2014ISAAC}]\label{theorem:general-drift-theorem}
Let $(X_t)_{t\in\mathbb{N}}$, be a stochastic process, adapted to a filtration $(\filt)_{t\in\mathbb{N}}$, over some state space $S\subseteq \{0\}\cup [\xmin,\xmax]$, where
$\xmin\ge 0$.
Let $g\colon \{0\}\cup [\xmin,\xmax]\to \R^{\ge 0}$ be any function
such that $g(0)=0,$ and $0<g(x)<\infty$ for all $x\in[\xmin,\xmax]$.
Let $T_a=\min\{t\mid X_t\le a\}$ for $a\in \{0\}\cup [\xmin,\xmax]$.
Then:
\begin{enumerate}[leftmargin=!,labelwidth=7mm]
\item[(i)]
If
$\expect{g(X_t)-g(X_{t+1})  \filtcond{ X_t\ge \xmin}}\ge \alpha_{\mathrm{u}}$
for some $\alpha_{\mathrm{u}}>0$ then
\[ \expect{T_0\mid X_0} \le \frac{g(X_0)}{\alpha_{\mathrm{u}}}. \]

\item[(ii)]
If
$\expect{g(X_t)-g(X_{t+1})  \filtcond{ X_t\ge \xmin}}\le \alpha_{\mathrm{\ell}}$
for some $\alpha_{\mathrm{\ell}}>0$ then
\[ \expect{T_0\mid X_0} \ge \frac{g(X_0)}{\alpha_{\mathrm{\ell}}}. \]
\item[(iii)]
If
there exists $\gamma>0$ and a function $\beta_{\mathrm{u}}\colon \mathbb{N}\to\R^+$
such that \[\expect{e^{-\gamma (g(X_t)-g(X_{t+1}))} \filtcond{ X_t >a}} \le \beta_{\mathrm{u}}(t)\]
 then for $t > 0$
\[\Prob{T_a>t \mid X_0} < \left(\prod_{r=0}^{t-1} \beta_{\mathrm{u}}(r)\right) \cdot  e^{\gamma (g(X_0)-g(a))}.\]

\item[(iv)]
If
there exists $\gamma>0$ and a function $\beta_{\mathrm{\ell}}\colon \mathbb{N}\to\R^+$
such that \[\expect{e^{\gamma (g(X_t)-g(X_{t+1}))}\filtcond{ X_t > a}} \le \beta_{\mathrm{\ell}}(t)\] then for $t > 0$,
\[\Prob{T_a < t \mid X_0>a} \le \left( \sum_{s=1}^{t-1} \prod_{r=0}^{s-1} \beta_{\mathrm{\ell}}(r)\right) \cdot  e^{-\gamma (g(X_0)-g(a))}.\]

If additionally the set of states $S\cap \{x\mid x\le a\}$ is absorbing, then\\
\[\Prob{T_a < t \mid X_0>a} \le \left(\prod_{i=0}^{t-1} \beta_{\mathrm{\ell}}(i)\right) \cdot e^{-\gamma (g(X_0)-g(a))}.\]
\end{enumerate}
\end{theorem}

\end{document}